\documentclass[11pt]{article}

\usepackage{fullpage}
\usepackage{amsmath,amssymb,amsthm}
\usepackage{mathtools}
\usepackage{booktabs}
\usepackage{tikz}
\usetikzlibrary{arrows.meta,calc,decorations.pathreplacing,positioning}
\colorlet{okblue}{blue!55!black}
\colorlet{nored}{red!65!black}
\tikzset{
  slab/.style={fill=okblue,fill opacity=0.09,draw=none},
  slabedge/.style={okblue,line width=0.6pt},
  qpath/.style={line width=1.3pt,nored},
  ghost/.style={black!45,dash pattern=on 1pt off 1.6pt,line width=0.8pt},
  lbl/.style={font=\footnotesize}}
\usepackage[colorlinks=true,linkcolor=blue,citecolor=blue,urlcolor=blue]{hyperref}
\usepackage[font={small,it}]{caption}

\theoremstyle{plain}
\newtheorem{theorem}{Theorem}[section]
\newtheorem{lemma}[theorem]{Lemma}
\newtheorem{proposition}[theorem]{Proposition}
\newtheorem{corollary}[theorem]{Corollary}
\theoremstyle{definition}
\newtheorem{definition}[theorem]{Definition}

\theoremstyle{remark}
\newtheorem{remark}[theorem]{Remark}
\newtheorem{question}[theorem]{Open Problem}

\DeclareMathOperator{\tr}{tr}
\DeclareMathOperator{\Reg}{Reg}

\DeclareMathOperator{\adj}{adj}
\DeclareMathOperator{\sign}{sign}
\DeclareMathOperator*{\argmax}{arg\,max}

\newcommand{\R}{\mathbb{R}}
\newcommand{\Z}{\mathbb{Z}}
\newcommand{\E}{\mathbb{E}}
\newcommand{\Ball}{\mathbb{B}}
\newcommand{\wstar}{w^{\ast}}
\newcommand{\ip}[2]{\langle #1, #2\rangle}

\title{\vspace{-1.5cm}\bf Optimal and Efficient Online Inverse Optimization}
\author{
\parbox{0.95\textwidth}{\centering
Anupam Gupta\thanks{Also with New York University.},
Guru Guruganesh,
Honghao Lin,\\
Vahab Mirrokni,
Renato Paes Leme,
David P.~Woodruff\thanks{Also with Carnegie Mellon University.}
}\\[1.5em]
\normalsize Google Research
}
\date{}

\begin{document}
\maketitle
\vspace{-1.2em}

\begin{abstract}
In online inverse linear optimization, a learner recommends an action and then observes
the choice of an expert who maximizes a fixed, unknown linear objective on $\R^{d}$; the goal is to
learn to optimize this objective without observing it. Sakaue recently obtained the optimal regret
$O(\sqrt d)$ with a randomized algorithm making $(dT)^{O(d)}$ linear optimizations per round, and
asked whether it can be attained in polynomial time. We answer positively: our deterministic
algorithm has regret $O(\sqrt d)$ for every horizon $T$ and runs in time polynomial in $d$ and $T$.
It is a variant of the variable-metric algorithms of Sakaue et al.\ and Cai et al., in which a
metric update is revoked once the query point moves far enough from where the update was made.
\end{abstract}

\section{Introduction}\label{sec:intro}

Consider a resident who learns by watching the decisions of an attending physician on a hospital
ward. Every patient comes with a handful of possible interventions --- this drug or another, operate
now or wait --- and the attending picks one. The reasoning behind the pick is seldom spelled out: it
balances how well a treatment works against what it does to the patient, and how fast recovery
comes against what it costs. The resident never sees the criterion the attending uses, only the
intervention chosen, one patient at a time. The goal is for the resident's recommendations to become
as good as the attending's in the long run: each recommendation falls short of the attending's
choice by some amount, measured by the attending's own criterion, and we want the total shortfall
over all patients to stay small. This is the problem of online inverse linear optimization.

We represent the attending's criterion by a fixed vector $\wstar$. Each intervention is a vector of
$d$ attributes, and the interventions available for the $t$-th patient form a set
$X_t\subseteq\R^{d}$. The resident proposes one, and is then shown the one the attending chose, a
maximizer of $\ip{\wstar}{\cdot}$ --- and nothing more: not the criterion, not what any option is
worth, not the size of the mistake. As in \cite{Cai26}, this leads to the following problem.

\smallskip
\noindent\textit{Online inverse linear optimization.} Fix a hidden $\wstar\in\Ball$. For
$t=1,2,\dots$:
\begin{enumerate}
\itemsep0pt \parskip0pt \topsep2pt
\item a compact nonempty action set $X_t\subseteq\Ball$ is revealed, chosen adversarially as a
      function of the history;
\item the learner recommends an action $\hat x_t\in X_t$;
\item the expert takes $x_t\in\argmax_{x\in X_t}\ip{\wstar}{x}$, which the learner then observes.
\end{enumerate}
\smallskip

Here $\Ball$ is the unit ball of $\R^{d}$ and $\|\cdot\|$ the Euclidean norm. The learner is charged
the cumulative shortfall of its recommendations, in the expert's own currency,
\begin{equation}\label{eq:regret}
  R_T\;:=\;\sum_{t=1}^{T}\ip{\wstar}{x_t-\hat x_t}\;\ge\;0 ,
\end{equation}
and never observes a single term of it. A bound is \emph{uniform in $T$} if it does not depend on
the horizon; it then bounds the total shortfall ever incurred. A learner is \emph{proper} if it
commits to an objective before it sees the menu: at the start of round $t$ it holds an estimate
$\hat w_t\neq0$ of $\wstar$ and recommends $\hat x_t\in\argmax_{x\in X_t}\ip{\hat w_t}{x}$, at the
cost of one linear optimization over $X_t$. It is \emph{improper} if $\hat x_t$ may be any function
of $X_t$ and the history. A proper learner carries an explicit hypothesis about the attending's
criterion, and each of its proposals is the best one under that hypothesis.

This paper is about the dimension. Bounds uniform in $T$ are now known (Table~\ref{tab:oilo}), so
the horizon is no longer the obstacle; what remains is how the total shortfall grows with $d$, the
number of attributes the criterion weighs. There is a floor: an adversary that offers the two
actions $\pm e_i$ in round $i\le d$ and draws the signs of $\wstar\in\{\pm1/\sqrt d\}^{d}$ at random
makes the better action of each round independent of everything seen before, so every learner,
even a randomized one, pays $\sqrt d$ in expectation \cite{SakaueONS,SakaueMMW}. We ask for a
learner that pays only that much and is also proper and \emph{efficient}, with
$\mathrm{poly}(d,T)$ arithmetic operations per round.

\begin{table}[t]
\centering\scriptsize\setlength{\tabcolsep}{4pt}
\begin{tabular}{@{}llllll@{}}
\toprule
Method & Regret & Uniform in $T$ & Proper & Deterministic & Per-round cost \\
\midrule
B\"armann et al.\ \cite{Barmann17,Barmann20} (OGD/MWU) & $O(\sqrt T)$ & no & yes & yes &
  $O(\tau_{\mathrm{sol}}+d)$ \\
Besbes et al.\ \cite{Besbes21,Besbes25} (circumcenter) & $O(d^{4}\ln T)$ & no & yes & yes &
  $\tau_{\mathrm{sol}}+\mathrm{poly}(d,T)$ \\
Gollapudi et al.\ \cite{GGKMPS} (centroid) & $O(d\ln T)$ & no & yes & yes &
  $\tau_{\mathrm{sol}}+\mathrm{poly}(d,T)$ \\
Gollapudi et al.\ \cite{GGKMPS} (John ellipsoid) & $\exp(O(d\ln d))$ & yes & yes & yes &
  $\tau_{\mathrm{sol}}+\mathrm{poly}(d,T)$ \\
Sakaue et al.\ \cite{SakaueONS} (online Newton step) & $O(d\ln T)$ & no & yes & yes &
  $O(\tau_{\mathrm{sol}}+d^{2}+\tau_{\mathrm{proj}})$ \\
Sakaue \cite{SakaueSOP} (second-order perceptron) & $O(d\ln T)$ & no & yes & yes &
  $O(\tau_{\mathrm{sol}}+d^{2})$ \\
Dewasurendra \cite{Dewasurendra26} (multiscale vote) & $O(d)$ & yes & no & yes & $T^{\Theta(d)}$ \\
Cai et al.\ \cite{Cai26} (variable metric) & $O(d)$ & yes & yes & yes &
  $O(\tau_{\mathrm{sol}}+d^{2})$ \\
Sakaue \cite{SakaueMMW} (multiscale matrix weights) & $O(\sqrt d)^{\dagger}$ & yes & no & no &
  $(dT)^{O(d)}\,\tau_{\mathrm{sol}}$ + finite \\
\midrule
\textbf{This paper} (revocable variable metric) & $\mathbf{O(\sqrt d)}$ & \textbf{yes} &
  \textbf{yes} & \textbf{yes} & $O(\tau_{\mathrm{sol}}+td^{2}+t^{2}d)$ \\
\quad with skipping (Prop.~\ref{prop:small}) & $\mathbf{O(\sqrt d)}$ & \textbf{yes} &
  \textbf{yes} & \textbf{yes} & amort.\ $O(\tau_{\mathrm{sol}}+d^{2}\log T)$ \\
\addlinespace
Lower bound \cite{SakaueONS}; \cite[Thm.~3.1]{Cai26} & $\Omega(\sqrt d)$ & --- & --- & --- & --- \\
\bottomrule
\end{tabular}
\caption{Regret $R_T$ of online inverse linear optimization on the unit ball. $\tau_{\mathrm{sol}}$:
one linear optimization over $X_t$; $\tau_{\mathrm{proj}}$: one Mahalanobis projection;
$^{\dagger}$in expectation.}
\label{tab:oilo}
\vspace{-1em}
\end{table}

\paragraph{History.} The online formulation is due to B\"armann, Pokutta and Schneider
\cite{Barmann17,Barmann20}, who obtained $O(\sqrt T)$ by online gradient descent and multiplicative
weights. Besbes, Fonseca and Lobel \cite{Besbes21,Besbes25} obtained $O(d^{4}\ln T)$ by a
circumcenter rule, ProjectedCones. Gollapudi, Guruganesh, Kollias, Manurangsi, Paes Leme and
Schneider \cite{GGKMPS}, who studied the same protocol as \emph{contextual recommendation},
obtained $O(d\ln T)$ by a centroid cutting-plane rule and $\exp(O(d\ln d))$, uniform in $T$, from
the John ellipsoid. Sakaue, Tsuchiya, Bao and Oki \cite{SakaueONS} obtained $O(d\ln T)$ at
$O(d^{2})$ per round plus a Mahalanobis projection by an online Newton step, together with the
lower bound above, and Sakaue \cite{SakaueSOP} kept $O(d\ln T)$ and removed the projection by a
second-order perceptron. Under structural assumptions the regret is finite: Sakaue, Bao and
Tsuchiya \cite{SakaueFY} assume that the gap in value of every action is at least a fixed multiple
of its distance to the optimal one, Oki and Sakaue \cite{OkiSakaue} that every $X_t$ is
$\mathrm{M}$-convex, as the bases of a matroid are, and Kitaoka \cite{Kitaoka26} obtains finite
regret and finitely many mistakes for the online Newton step on uniformly bounded integer action
sets, for $\wstar$ in the probability simplex and unique optimal actions; that bound carries a
factor $d^{2}$ and depends on the range of the coordinates. Whether a finite $\mathrm{poly}(d)$
bound holds for arbitrary action sets was asked in \cite{GGKMPS,OkiSakaue}. Dewasurendra
\cite{Dewasurendra26} answered it with $O(d)$, uniform in $T$: at every dyadic scale he covers the
class of optimality-gap functions, pools the resulting tolerant tests into one weighted vote and
plays the action the vote prefers --- an improper rule with $T^{\Theta(d)}$ tests per round. Cai,
Gupta, Gupta, Guruganesh, Jiang, Liaw, Mehta, Paes Leme, Velegkas and Wang \cite{Cai26} obtained
$O(d)$, precisely $\Reg_T\le6\sqrt3\,dR$ in the game below, with a deterministic proper rule at
$O(d^{2})$ per round: a variable-metric step with a self-normalized rank-one update, analyzed
through the trace power $\tr(H^{-1/2})$. Sakaue \cite{SakaueMMW} then reached the optimal order,
$\E[R_T]\le2^{21}\sqrt d$ for every $T$, without knowledge of $T$ and for arbitrary compact action
sets and ties, by matrix multiplicative weights on polynomial feature spaces at geometrically spaced
scales. The learner is randomized and improper --- it solves a linear program for a distribution
over a list of actions and samples from it --- and in round $t$ its finest scale acts on polynomials
of degree $\Theta(t^{4})$ in $d$ variables, with $(dT)^{O(d)}$ linear optimizations over $X_t$ in
its rational implementation.
Sakaue proves that every round terminates, gives no polynomial bound on the running time, and asks
whether $\sqrt d$ is attainable in time polynomial in the dimension, the horizon and the input
length \cite[Section~7]{SakaueMMW}. So far, then, $\sqrt d$ has been reached only by an
improper, randomized rule with no polynomial bound, and $O(d)$ is the best bound known for a proper
learner, for a deterministic one, and for an efficient one.

The problem is the online face of an older one --- recover an objective from observed optimal
decisions --- initiated by Burton and Toint \cite{BurtonToint} for shortest paths and cast in
linear-programming form by Ahuja and Orlin \cite{AhujaOrlin}; see \cite{ChanSurvey} for a survey.
Offline work turns on the choice of loss \cite{Keshavarz11,Bertsimas15,MohajerinEsfahani18,
Zattoni25}, on noisy \cite{Aswani18} or partially specified \cite{Ren25} data, and on statistical
rates \cite{Fatemi26}; applications run from control \cite{Akhtar21} and routing
\cite{ZattoniRouting} to offline reinforcement learning \cite{Dimanidis25}.

\paragraph{Reduction to a cutting-plane game.} As in \cite{Cai26}, we work with the following
primitive, isolated in \cite{GGKMPS}.

\smallskip
\noindent\textit{Cutting planes with a strong separation oracle.} A point $\wstar$ is hidden in a
known compact convex body $K_1\subseteq\R^{d}$; let $R:=\max_{w\in K_1}\|w\|$. For $t=1,2,\dots$:
\begin{enumerate}
\itemsep0pt \parskip0pt \topsep2pt
\item the learner queries a point $p_t\in\R^{d}$;
\item an adversary, having seen $p_t$ and the history, returns a unit vector $v_t$ subject only to
      $r_t:=\ip{\wstar-p_t}{v_t}\ge0$.
\end{enumerate}
The regret is $\Reg_T:=\sum_{t\le T}r_t$, a sum of nonnegative losses; the oracle is \emph{strong}
in that $v_t$ may be chosen after seeing $p_t$.
\smallskip

A learner for this game on $K_1=\Ball$ gives a proper learner for online inverse linear
optimization: play $\hat w_t:=p_t$ (any fixed nonzero vector if $p_t=0$), recommend
$\hat x_t\in\argmax_{x\in X_t}\ip{\hat w_t}{x}$, and return $v_t:=\delta_t/\|\delta_t\|$ for
$\delta_t:=x_t-\hat x_t$, skipping the rounds with $\delta_t=0$, which cost nothing. Optimality of
$\hat x_t$ and of $x_t$ gives $\ip{p_t}{\delta_t}\le0\le\ip{\wstar}{\delta_t}$, so
$r_t=\ip{\wstar-p_t}{\delta_t}/\|\delta_t\|\ge0$ and the round is legal, and as $\|\delta_t\|\le2$
the shortfall is $\ip{\wstar}{\delta_t}\le\ip{\wstar-p_t}{\delta_t}=\|\delta_t\|\,r_t\le2r_t$. Hence
$R_T\le2\Reg_T$; the expert need only be, round by round, at least as good as the recommendation
under its own criterion.

\subsection{Our Results}\label{sec:results}

We solve the cutting-plane game at the optimal rate, deterministically and in polynomial time.

Our main result is an algorithm which, given only the radius $R=\max_{w\in K_1}\|w\|$, has regret
$\Reg_T=O(R\sqrt d)$ for every horizon $T$, every hidden $\wstar\in K_1$
and every strong separation oracle, and whose round $t$ takes $O(td^{2}+t^{2}d)$ arithmetic
operations, $O(d^{2}+Td)$ per round on average over the first $T$ rounds
(Theorem~\ref{thm:rvm}). The bound is uniform in $T$, hence holds for $T=\infty$, and the algorithm
is deterministic and anytime: it needs neither $T$ nor restarts. Through the reduction above, with
$K_1=\Ball$, it gives a deterministic proper learner for online inverse linear optimization with
$R_T=O(\sqrt d)$ for every horizon and every adaptive sequence of action sets,
using one linear optimization over $X_t$ and the same number of arithmetic operations per round
(Corollary~\ref{cor:oilo}).

Both bounds are optimal: every learner, deterministic or randomized, suffers regret $\Omega(R\sqrt
d)$ in the game on $K_1=R\Ball$ \cite[Theorem~3.1]{Cai26}, and $\Omega(\sqrt d)$ in expectation in
online inverse linear optimization for every $T\ge d$ \cite{SakaueONS,SakaueMMW}, so the optimal
regret is $\Theta(R\sqrt d)$ and $\Theta(\sqrt d)$, respectively. Ours is the first $\sqrt d$ bound
for online inverse linear optimization that is deterministic, the first that is proper, and the
first with a polynomial bound on its computation, which answers the question of \cite{SakaueMMW}.
Our cost exceeds the $O(d^{2})$ of \cite{Cai26} and grows with $t$. A variant that skips every
stretch with $s_t^{2}\le(t+1)^{-4}$ keeps only $O(d\log t)$ stretches active, at an amortized
$O(d^{2}\log T)$ arithmetic operations per round and $O(d^{3}\log^{2}t)$ in the worst round, with
regret still $O(R\sqrt d)$ (Proposition~\ref{prop:small}).

The main body works in the real-RAM model, counting arithmetic operations on real numbers, and it is
self-contained: nothing in it depends on the appendix. Appendix~\ref{sec:bits} counts bit
operations instead: a rounded variant of the algorithm, which reads each reply to finite precision,
keeps regret $O(R\sqrt d)$ and uses a number of bit operations per round polynomial in $d$, $\log t$
and the encoding length of $R$ (Theorem~\ref{thm:bits}).

\paragraph{The algorithm.} Our starting point is the variable-metric algorithm introduced for this
problem by Sakaue, Tsuchiya, Bao and Oki \cite{SakaueONS} and Sakaue \cite{SakaueSOP}, and later
made horizon-independent by Cai et al.\ \cite{Cai26}. In the version of \cite{Cai26}, from $p_1=0$
and $H_1=I$ it updates
\begin{equation}\label{eq:vm}
  s_t=\|v_t\|_{H_t^{-1}},\qquad H_{t+1}=H_t+\frac{\tau}{s_t}\,v_tv_t^{\top},\qquad
  p_{t+1}=p_t+\alpha H_{t+1}^{-1}v_t ,
\end{equation}
with $\tau=\Theta(1/d)$ and $\alpha=\Theta(R/d)$, where $\|x\|_H:=(x^{\top}Hx)^{1/2}$.
It is a \emph{variable-metric} method: the positive definite matrix $H_t$ is a metric that records
what has been learned about $\wstar$, and each round \emph{stretches} it along the direction $v_t$
just received, to indicate that more is now known along $v_t$; the steps $H^{-1}v$ are then long in
directions not yet explored and short in explored ones. The analysis of \cite{Cai26} controls the
distance $\|p_t-\wstar\|_{H_t}$ in the current metric. Every stretch inflates that distance and is
never undone: this forces us to be conservative and maintain small stretches---i.e., to keep $\tau=\Theta(1/d)$---which ends up incurring a regret that is 
$\sqrt d$ larger than the best possible.

We also maintain a metric, but a stretch is no longer permanent: it is \emph{revoked} as soon as the query
moves far along its direction from the query point at which it was made. The learner keeps a set
$A_t$ of \emph{active} stretches, and its metric is
\[
  H_t\;=\;I+\sum_{i\in A_t}\gamma_i\,v_iv_i^{\top} .
\]
The stretch $i$ made at the end of round $i$ specifies a weight $\gamma_i$ and an associated \emph{slab} $L_i$ (consisting of the
set of points within $4\alpha s_i^{2}$ of $p_{i+1}$ along the direction $v_i$); stretch $i$ remains active only while the
query belongs to slab $L_i$. In round $t$ the query moves continuously along $H^{-1}v_t$ for a time $\alpha$, where 
$H$ corresponds to the currently active stretches; the moment the query leaves the slab of an active
stretch, that stretch is revoked and the direction changes correspondingly. At the end of the round a new stretch is
added to $A_t$. Formally, the algorithm is the following.

\begin{definition}[Algorithm RVM$(\alpha)$]\label{def:alg}
Start from $p_1=0$ and $A_1=\emptyset$. In round $t=1,2,\dots$, query $p_t$ and receive $v_t$. Then
follow the continuous path $\hat p_t:[0,\alpha]\to\R^{d}$ with right derivative given by
\begin{equation}\label{eq:move}
\begin{gathered}
  \hat p_t(0)=p_t,\qquad \hat p_t'(\eta)=\hat H_t(\eta)^{-1}v_t,\qquad
  \hat H_t(\eta)=I+\sum_{i\in\hat A_t(\eta)}\gamma_i\,v_iv_i^{\top},\\
  \hat A_t(\eta)=\bigl\{i\in A_t:\ \hat p_t(\tilde\eta)\in L_i\
    \text{for all}\ 0\le\tilde\eta\le\eta\bigr\},
\end{gathered}
\end{equation}
and set
\begin{equation}\label{eq:stretch}
\begin{gathered}
  p_{t+1}=\hat p_t(\alpha),\qquad s_t=\|v_t\|_{\hat H_t(\alpha)^{-1}},\qquad
  \gamma_t=\frac{1}{8s_t^{2}},\qquad
  L_t=\bigl\{p\in\R^{d}:\ |\ip{p-p_{t+1}}{v_t}|<4\alpha s_t^{2}\bigr\},\\
  A_{t+1}=\hat A_t(\alpha)\cup\{t\},\qquad H_{t+1}=\hat H_t(\alpha)+\gamma_t\,v_tv_t^{\top} .
\end{gathered}
\end{equation}
We run it with $\alpha=R/(3\sqrt d)$.
\end{definition}

The term $\hat A_t(\eta)$ is the set of
stretches whose slab the path has not yet left, so a stretch is revoked the moment the path reaches
the boundary of its slab, at $\eta=\alpha$ included, and never returns; without revocations the path
is a single segment and $p_{t+1}=p_t+\alpha H_t^{-1}v_t$, a variable-metric step like that of
\eqref{eq:vm}. By
Lemma~\ref{lem:wellposed} the path is well defined and piecewise linear, with at most $|A_t|+1$
pieces.
% The number $s_t\in(0,1]$ is the dual length of $v_t$ in the metric left after the move. The
% weight makes $\gamma_ts_t^{2}=\tfrac18$, so by Sherman--Morrison the stretch lowers
% $\|v_t\|^{2}_{H^{-1}}$ from $s_t^{2}$ to $\tfrac89s_t^{2}$: a constant-factor cut, as in the
% ellipsoid method, but of the metric rather than of a feasible set. Compared with \eqref{eq:vm}, the
% stretch has the same self-normalized form, with weight $\tau/s_t^{2}$ at the constant
% $\tau=\tfrac18$ in place of $\tau/s_t$ at $\tau=\Theta(1/d)$, and it comes with a slab that can
% revoke it. The half-width $4\alpha s_t^{2}$ of $L_t$ is four times the largest distance
% $\alpha s_t^{2}$ that the query can travel along $v_t$ in round $t$ (Lemma~\ref{lem:path}), and the
% product $\gamma_t\cdot4\alpha s_t^{2}=\alpha/2$ of weight and half-width is the same for every
% stretch; these two facts are what the analysis uses.

Every quantity the update uses is observable
--- the queries, the directions received and the stretches made from them --- and neither the loss
$r_t$ nor $\wstar$ appears. The learner is proper, since its queries are its estimates, and uses no
randomness. Keeping $H^{-1}$ by Sherman--Morrison, round $t$ costs $O((k_t+1)(d^{2}+|A_t|d))$
operations, where $k_t$ is the number of stretches it revokes (Remark~\ref{rem:cost}); as
$|A_t|\le t-1$ and each stretch is revoked at most once, this gives the counts of
Theorem~\ref{thm:rvm}.

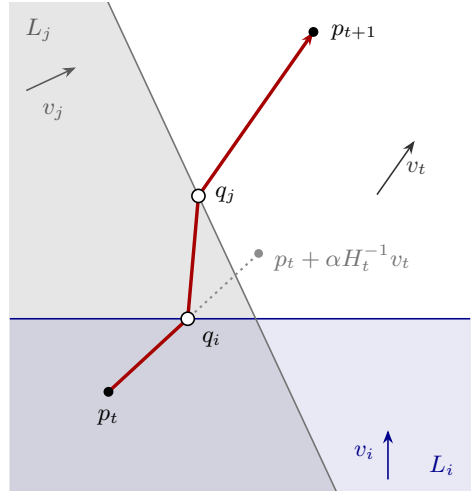
\begin{figure}[t]
\centering
\begin{tikzpicture}[scale=1.45,>={Stealth[length=5pt]}]
  % alpha = 1, drawn at 8x. Stretch i: v_i = e_2, s_i^2 = 1/24 (gamma_i = 3, half-width 1/6).
  % Stretch j: v_j at 25 degrees, s_j^2 = 1/16 (gamma_j = 2, half-width 1/4). v_t at 55 degrees.
  \begin{scope}
    \clip (-0.9,-0.9) rectangle (3.3,3.55);
    \fill[black!10] ($(1.359,0.634)+10*(-0.4226,0.9063)$) -- ($(1.359,0.634)-10*(-0.4226,0.9063)$)
       -- ($(-2.265,-1.056)-10*(-0.4226,0.9063)$) -- ($(-2.265,-1.056)+10*(-0.4226,0.9063)$)
       -- cycle;
    \fill[slab] (-0.9,-2.0) rectangle (3.3,0.667);
    \draw[slabedge] (-0.9,0.667) -- (3.3,0.667);
    \draw[black!55,line width=0.6pt] (2.727,-2.3) -- (-0.024,3.6);
  \end{scope}
  \node[lbl,okblue,anchor=south east] at (3.25,-0.85) {$L_i$};
  \draw[->,okblue,line width=0.6pt] (2.55,-0.8) -- (2.55,-0.35);
  \node[lbl,okblue,anchor=east] at (2.52,-0.55) {$v_i$};
  \node[lbl,black!65,anchor=north west] at (-0.85,3.5) {$L_j$};
  \draw[->,black!65,line width=0.6pt] (-0.75,2.75) -- ($(-0.75,2.75)+(25:0.5)$);
  \node[lbl,black!65,anchor=north] at (-0.5,2.72) {$v_j$};
  \draw[->,black!80,line width=0.6pt] (2.45,1.8) -- ($(2.45,1.8)+(55:0.6)$);
  \node[lbl,black!80,anchor=west] at ($(2.45,1.8)+(55:0.3)$) {$v_t$};
  % the single step with the metric frozen
  \draw[ghost] (0,0) -- (1.370,1.263);
  \fill[black!45] (1.370,1.263) circle (0.04);
  \node[lbl,black!55,anchor=west] at (1.42,1.22) {$p_t+\alpha H_t^{-1}v_t$};
  % the path of the move
  \draw[qpath] (0,0) -- (0.723,0.667) -- (0.821,1.788);
  \draw[qpath,->] (0.821,1.788) -- (1.870,3.285);
  \fill (0,0) circle (0.045);
  \node[lbl,anchor=north] at (0,-0.07) {$p_t$};
  \draw[line width=0.6pt,fill=white] (0.723,0.667) circle (0.06);
  \node[lbl,anchor=north west] at (0.76,0.64) {$q_i$};
  \draw[line width=0.6pt,fill=white] (0.821,1.788) circle (0.06);
  \node[lbl,anchor=west] at (0.88,1.8) {$q_j$};
  \fill (1.870,3.285) circle (0.045);
  \node[lbl,anchor=west] at (1.94,3.285) {$p_{t+1}$};
\end{tikzpicture}
\caption{One move of RVM$(\alpha)$ in the plane, to scale ($\alpha=1$). The path (red) leaves the
slab $L_i$ at $q_i$ and the slab $L_j$ at $q_j$, turning at each revocation; the dotted segment is
the single step with the metric frozen.}
\label{fig:round}
\end{figure}

Figure~\ref{fig:round} gives the geometric picture. Two stretches $i$ and $j$ are active at $p_t$,
and each shortens the step along its own direction, so the path leaves $p_t$ bent away from $v_t$.
At $q_i$ it reaches the edge of $L_i$: stretch $i$ is revoked, the metric loses the 
$\gamma_iv_iv_i^{\top}$ term, and the path turns correspondingly. At $q_j$ it leaves $L_j$, stretch $j$ is revoked too,
and the path continues along $v_t$ itself until time $\alpha$; stretch $t$ is then
defined, with its slab centered at $p_{t+1}$ (its half-width $4\alpha s_t^{2}=4$ is off the picture). As a comparison the dotted segment is the single step $p_t+\alpha H_t^{-1}v_t$ taken by \cite{Cai26}.
% \anupam{The dotted segment is the single step $p_t+\alpha H_t^{-1}v_t$ with the metric frozen. It ends
% outside both slabs, so revoking there would refund the two stretches at points beyond their
% half-widths, which the analysis cannot afford (Remark~\ref{rem:why}).}
We note that the slabs are not knowledge
sets: no point is ruled out, and $\wstar$ may lie anywhere.

\paragraph{Organization.} Section~\ref{sec:analysis} states the results precisely and proves them.
After checking that the move is well defined and counting its cost, it bounds the change of
$\|p-\wstar\|_{H}^{2}$ over one move (Lemma~\ref{lem:path}), shows that over the lifetime of a stretch
its penalty is refunded up to a term linear in the loss (Lemma~\ref{lem:lifetime}), and pays for the
stretches that are never revoked with a trace potential (Lemma~\ref{lem:potential}).
Section~\ref{sec:further} shows why revocation is needed and why it happens exactly on the slab
boundary, and gives a variant with $O(d\log t)$ active stretches. Section~\ref{sec:open} concludes.
Appendix~\ref{sec:bits} treats finite precision and bit complexity.

\section{Analysis}\label{sec:analysis}

We state the main results precisely and prove them. The first is the regret bound and cost of
RVM$(\alpha)$.

\begin{theorem}[Regret of RVM$(\alpha)$]\label{thm:rvm}
Let $R=\max_{w\in K_1}\|w\|$ and run RVM$(\alpha)$ with $\alpha=R/(3\sqrt d)$. Then for every $T$,
every $\wstar\in K_1$ and every strong separation oracle,
\[
  \Reg_T\;\le\;\frac{51}{8}\,R\sqrt d\;<\;6.4\,R\sqrt d
  \qquad\text{and}\qquad
  \|p_{T+1}-\wstar\|^{2}\;\le\;\frac{17}{8}\,R^{2} .
\]
Round $t$ uses $O(td^{2}+t^{2}d)$ arithmetic operations, and the first $T$ rounds use
$O(d^{2}+Td)$ per round on average.
\end{theorem}

Through the reduction of Section~\ref{sec:intro}, the theorem gives our result for online inverse
linear optimization.

\begin{corollary}[Online inverse optimization]\label{cor:oilo}
There is a deterministic proper learner for online inverse linear optimization with
\[
  R_T\;\le\;\frac{51}{4}\sqrt d\;<\;13\sqrt d
\]
for every $T$, every $\wstar\in\Ball$ and every adaptive sequence of action sets. Its round $t$
uses one linear optimization over $X_t$ and $O(td^{2}+t^{2}d)$ arithmetic operations, and its first
$T$ rounds use $O(d^{2}+Td)$ arithmetic operations per round on average.
\end{corollary}

\begin{proof}
Run RVM$(\alpha)$ on $K_1=\Ball$, so $R=1$, through the reduction of Section~\ref{sec:intro}, which
gives $R_T\le2\Reg_T$ and adds one linear optimization and $O(d)$ arithmetic operations per round,
one of them the square root that normalizes $v_t$. Rounds with $x_t=\hat x_t$ are not rounds of the
game, so the game's round index never exceeds $t$, and Theorem~\ref{thm:rvm} applies.
\end{proof}

The cost bounds of Theorem~\ref{thm:rvm} are counted in Remark~\ref{rem:cost}, once
Lemma~\ref{lem:wellposed} has checked that the move is well defined; the rest of the section proves
the regret bound.

The proof tracks the squared distance from the query to $\wstar$ in the metric in use,
$\|p-\wstar\|_{H}^{2}$ --- the quantity $\|p_t-\wstar\|_{H_t}^{2}$ of \cite{Cai26} --- through the
three kinds of events that change it: the moves, the making of stretches, and their revocation.
Lemma~\ref{lem:path} accounts for a move, \eqref{eq:ledger} regroups the whole run by stretch,
Lemma~\ref{lem:lifetime} bounds the contribution of each stretch over its lifetime, and
Lemma~\ref{lem:potential} bounds the total left over by the stretches that are never revoked.

\subsection{The move and its cost}\label{sec:wellposed}

The move \eqref{eq:move} defines the path and the set of active stretches in terms of each other.
The first lemma checks that exactly one path fits the definition and collects the facts about it
that the analysis uses.

\begin{lemma}[Well-posed move]\label{lem:wellposed}
In every round $t$ of RVM$(\alpha)$:
\begin{enumerate}
\itemsep0pt
\item[(a)] the slab of every active stretch contains the query: $p_t\in L_i$ for all $i\in A_t$, so
      that $\hat A_t(0)=A_t$ and $\hat H_t(0)=H_t$;
\item[(b)] exactly one path satisfies \eqref{eq:move}; it is piecewise linear with at most
      $k_t+1\le|A_t|+1$ pieces, where $k_t:=|A_t\setminus\hat A_t(\alpha)|$ is the number of
      stretches revoked in the round;
\item[(c)] $\hat H_t(\eta)$ is nonincreasing in $\eta$, so the dual length
      $\sigma(\eta):=\|v_t\|_{\hat H_t(\eta)^{-1}}$ is nondecreasing, and $0<\sigma(\eta)\le s_t\le1$;
\item[(d)] each stretch $i$ revoked in the round is revoked at a point $q_i$ of the path on the
      boundary of its slab: $|\ip{q_i-p_{i+1}}{v_i}|=4\alpha s_i^{2}$.
\end{enumerate}
\end{lemma}

\begin{proof}
(a) By induction on $t$; at $t=1$ there are no stretches. Every $i\in\hat A_t(\alpha)$ has
$p_{t+1}=\hat p_t(\alpha)\in L_i$ by the definition of $\hat A_t(\alpha)$, and $p_{t+1}$ is the
center of $L_t$, so $p_{t+1}\in L_i$ for all $i\in A_{t+1}$.

(b) While the active set is fixed, the right derivative $\hat H_t(\eta)^{-1}v_t$ is constant, so the
path is a segment. By (a) the path starts with every stretch of $A_t$ active, and it follows a
segment until it reaches the boundary of an active slab or $\eta=\alpha$; the stretches on that
boundary are revoked, the others are still strictly inside their slabs, and a new segment begins.
Every piece but the last revokes a stretch, so there are at most $k_t+1$ pieces. The path is unique: since the slabs are open, any path
satisfying \eqref{eq:move} keeps its active set, hence its velocity, for a while after each time,
so two such paths that agree up to some time agree a little beyond it.

(c) The set $\hat A_t(\eta)$ only loses elements, and each loss subtracts a positive semidefinite
$\gamma_iv_iv_i^{\top}$ from the metric. Inversion reverses the Loewner order, so
$\hat H_t(\eta)^{-1}$ is nondecreasing, and so is $\sigma(\eta)^{2}=v_t^{\top}\hat
H_t(\eta)^{-1}v_t$. Finally $\hat H_t(\eta)\succeq I$ gives $0\prec\hat H_t(\eta)^{-1}\preceq I$,
hence $0<\sigma(\eta)\le\sigma(\alpha)=s_t\le\|v_t\|=1$.

(d) The path starts in the open slab $L_i$ by (a) and is continuous, so the first point at which it
is not in $L_i$, where stretch $i$ is revoked, lies on the boundary of $L_i$.
\end{proof}

\begin{remark}[Cost]\label{rem:cost}
The algorithm stores $M=H^{-1}$ for the metric in use, and, for each active stretch $i$, the
vector $v_i$, the numbers $\gamma_i,s_i$ and the displacement $\phi_i=\ip{p-p_{i+1}}{v_i}$ of the
current point. On a piece of the path it computes $u=Mv_t$ in $O(d^{2})$ operations and
$\beta_i=\ip{v_i}{u}$ for every active stretch in $O(|A_t|d)$; the piece ends at the least of the
remaining time and the exit times $(4\alpha s_i^{2}-\sign(\beta_i)\phi_i)/|\beta_i|$
over $\beta_i\neq0$, after which $p$ and the $\phi_i$ are advanced in $O(d+|A_t|)$. Revoking
stretch $i$ is the Sherman--Morrison downdate
\[
  M\;\leftarrow\;M+\frac{\gamma_i}{1-\gamma_iv_i^{\top}Mv_i}\,Mv_iv_i^{\top}M ,
\]
whose denominator equals $1/(1+\gamma_iv_i^{\top}N^{-1}v_i)>0$ for $N=H-\gamma_iv_iv_i^{\top}\succeq
I$ the metric after the downdate; making stretch $t$ is the update
$M\leftarrow M-\tfrac{1}{9s_t^{2}}Mv_tv_t^{\top}M$, since $\gamma_t/(1+\gamma_ts_t^{2})=1/(9s_t^{2})$.
Each costs $O(d^{2})$. By Lemma~\ref{lem:wellposed}(b), round $t$ therefore costs
\[
  O\bigl((k_t+1)(d^{2}+|A_t|\,d)\bigr)\;=\;O\bigl(td^{2}+t^{2}d\bigr)
\]
arithmetic operations, as $k_t\le|A_t|\le t-1$. Since each stretch is revoked at most once,
$\sum_{t\le T}k_t\le T$, so the first $T$ rounds cost $O(Td^{2}+T^{2}d)$ in total, that is,
$O(d^{2}+Td)$ per round on average, with memory $O(d^{2}+Td)$. No random bits are used. Apart from
the constant $\alpha$, all operations are rational, because $s_t$ enters only through
$s_t^{2}=v_t^{\top}Mv_t$. The counts are in the real-RAM model; Appendix~\ref{sec:bits} gives a
rounded variant with polynomial bit complexity.
\end{remark}

\subsection{One move}\label{sec:move}

Fix $T$, $\wstar$ and the oracle. In round $t$, $\xi(\eta):=\ip{\hat p_t(\eta)-p_t}{v_t}$ is the
distance the query has moved along $v_t$ by time $\eta$ of the path, and
$\xi_t:=\xi(\alpha)=\ip{p_{t+1}-p_t}{v_t}$. For the stretch made in round $t\le T$, we let $q_t$ be
the point at which it is revoked if that happens by the end of round $T$ and $q_t:=p_{T+1}$
otherwise, and we put
\[
  \Delta_t\;:=\;\ip{q_t-p_{t+1}}{v_t} ,
\]
its displacement along $v_t$ from the point where it was made. The stretches that are not revoked by
the end of round $T$ form the set $A_{T+1}$, and we call them \emph{surviving}. By
Lemma~\ref{lem:wellposed}(a) and (d),
\begin{equation}\label{eq:delta}
  |\Delta_t|\;=\;4\alpha s_t^{2}\ \ \text{if stretch $t$ is revoked},\qquad
  |\Delta_t|\;<\;4\alpha s_t^{2}\ \ \text{if it survives} .
\end{equation}

Along the path the query moves along $\hat H_t(\eta)^{-1}v_t$, which decreases the distance to
$\wstar$ at the rate $2\ip{\wstar-\hat p_t(\eta)}{v_t}$, and each revocation removes one term from
the metric. The lemma records the net effect, together with the fact that the query cannot move far
along $v_t$ in one round.

\begin{lemma}[One move]\label{lem:path}
In every round $t$, $0\le\xi_t\le\alpha s_t^{2}$, and
\[
  \|p_{t+1}-\wstar\|_{\hat H_t(\alpha)}^{2}-\|p_t-\wstar\|_{H_t}^{2}
  \;\le\;-2\alpha r_t+\alpha^{2}s_t^{2}
  -\sum_{i\in A_t\setminus\hat A_t(\alpha)}\gamma_i\ip{\wstar-q_i}{v_i}^{2} .
\]
\end{lemma}

\begin{proof}
On a piece of the path the metric $H=\hat H_t(\eta)$ is fixed and $\hat p_t'=H^{-1}v_t$, so
\[
  \frac{d}{d\eta}\,\|\hat p_t(\eta)-\wstar\|_{H}^{2}\;=\;2\ip{\hat p_t(\eta)-\wstar}{H\hat p_t'(\eta)}
  \;=\;-2\ip{\wstar-\hat p_t(\eta)}{v_t}\;=\;-2\bigl(r_t-\xi(\eta)\bigr) ,
\]
using $\ip{\wstar-\hat p_t(\eta)}{v_t}=\ip{\wstar-p_t}{v_t}-\ip{\hat p_t(\eta)-p_t}{v_t}$. Here
$\xi(0)=0$ and $\xi'(\eta)=\ip{\hat p_t'(\eta)}{v_t}=v_t^{\top}\hat H_t(\eta)^{-1}v_t=\sigma(\eta)^{2}
\in(0,s_t^{2}]$ by Lemma~\ref{lem:wellposed}(c), so $0\le\xi(\eta)\le\eta s_t^{2}$; at $\eta=\alpha$
this is the first claim. The rate of change is therefore at most $-2r_t+2\eta s_t^{2}$, and
integrating over $[0,\alpha]$ gives $-2\alpha r_t+\alpha^{2}s_t^{2}$. At the end of a piece, revoking
stretch $i$ at $q_i=\hat p_t(\eta)$ subtracts $\gamma_i\ip{\wstar-q_i}{v_i}^{2}$ from the squared
distance, which gives the last term.
\end{proof}

If no stretch is revoked, $\sigma(\eta)=s_t$ throughout and the lemma holds with equality: it is the
familiar expansion
$\|p_t+\alpha H_t^{-1}v_t-\wstar\|_{H_t}^{2}=\|p_t-\wstar\|_{H_t}^{2}-2\alpha r_t+\alpha^{2}s_t^{2}$.

\paragraph{The ledger.} Making stretch $t$ at $p_{t+1}$ adds $\gamma_t\ip{\wstar-p_{t+1}}{v_t}^{2}$
to the squared distance: $\|p_{t+1}-\wstar\|_{H_{t+1}}^{2}=\|p_{t+1}-\wstar\|_{\hat H_t(\alpha)}^{2}
+\gamma_t\ip{\wstar-p_{t+1}}{v_t}^{2}$. Add this to Lemma~\ref{lem:path} and sum over $t\le T$. Each
stretch revoked by the end of round $T$ is revoked exactly once, so its refund
$\gamma_t\ip{\wstar-q_t}{v_t}^{2}$ appears exactly once. The surviving stretches have $q_t=p_{T+1}$,
so $\|p_{T+1}-\wstar\|_{H_{T+1}}^{2}=\|p_{T+1}-\wstar\|^{2}+\sum_{t\in
A_{T+1}}\gamma_t\ip{\wstar-q_t}{v_t}^{2}$, and $p_1=0$, $H_1=I$ give
$\|p_1-\wstar\|_{H_1}^{2}=\|\wstar\|^{2}$. The penalty and the refund of stretch $t$ combine: since
$\ip{\wstar-p_{t+1}}{v_t}=r_t-\xi_t$ and $\ip{\wstar-q_t}{v_t}=r_t-\xi_t-\Delta_t$,
\[
  \ip{\wstar-p_{t+1}}{v_t}^{2}-\ip{\wstar-q_t}{v_t}^{2}
  \;=\;\Delta_t\bigl(2(r_t-\xi_t)-\Delta_t\bigr)\;=\;2\Delta_tr_t-\Delta_t(\Delta_t+2\xi_t) ,
\]
and the $r_t^{2}$ cancels. Collecting,
\begin{equation}\label{eq:ledger}
  \|p_{T+1}-\wstar\|^{2}-\|\wstar\|^{2}\;\le\;\sum_{t=1}^{T}\Lambda_t,\qquad
  \Lambda_t\;:=\;-2\alpha r_t+\alpha^{2}s_t^{2}
  +\gamma_t\bigl(2\Delta_tr_t-\Delta_t(\Delta_t+2\xi_t)\bigr) .
\end{equation}
We read $\Lambda_t$ as the \emph{lifetime contribution} of stretch $t$: the move that preceded it,
and the penalty $\gamma_t\ip{\wstar-p_{t+1}}{v_t}^{2}$ charged when the stretch was made, net of the
refund $\gamma_t\ip{\wstar-q_t}{v_t}^{2}$ collected when it is revoked --- or, for a surviving
stretch, still owed at the end.

\subsection{The lifetime of a stretch}\label{sec:lifetime}

In \cite{Cai26} the penalty for stretching the metric is never refunded, and it is quadratic in the
loss: $\tfrac{\tau}{s_t}r_t^{2}$ there, $\gamma_t\ip{\wstar-p_{t+1}}{v_t}^{2}=\gamma_t(r_t-\xi_t)^{2}$
here. The point of revocation is that the refund cancels the quadratic part exactly, as the ledger
shows, leaving in $\Lambda_t$ a term linear in $r_t$ whose coefficient the slab controls.

\begin{lemma}[Lifetime of a stretch]\label{lem:lifetime}
For every $t\le T$,
\[
  \Lambda_t\;\le\;-\alpha r_t\ \ \text{if stretch $t$ is revoked},\qquad
  \Lambda_t\;\le\;-\alpha r_t+\tfrac98\,\alpha^{2}s_t^{2}\ \ \text{if it survives} .
\]
\end{lemma}

\begin{proof}
Group the terms of $\Lambda_t$ by their dependence on the loss:
\[
  \Lambda_t\;=\;\underbrace{-2\alpha r_t+2\gamma_t\Delta_tr_t}_{\text{linear in }r_t}
  \;+\;\underbrace{\alpha^{2}s_t^{2}-\gamma_t\Delta_t(\Delta_t+2\xi_t)}_{\text{free of }r_t} .
\]
\emph{The linear part.} By \eqref{eq:delta}, $|\Delta_t|\le4\alpha s_t^{2}$, and
$\gamma_t\cdot4\alpha s_t^{2}=\alpha/2$; as $r_t\ge0$, $2\gamma_t\Delta_tr_t\le\alpha r_t$, so the
linear part is at most $-\alpha r_t$.

\emph{The rest, for a revoked stretch.} Here $|\Delta_t|=4\alpha s_t^{2}$, and
$0\le\xi_t\le\alpha s_t^{2}$ by Lemma~\ref{lem:path}, so
\[
  \gamma_t\Delta_t(\Delta_t+2\xi_t)\;\ge\;\gamma_t|\Delta_t|\bigl(|\Delta_t|-2\xi_t\bigr)
  \;\ge\;\frac{1}{8s_t^{2}}\cdot4\alpha s_t^{2}\bigl(4\alpha s_t^{2}-2\alpha s_t^{2}\bigr)
  \;=\;\alpha^{2}s_t^{2} .
\]
\emph{The rest, for a surviving stretch.} Here we only use
$\Delta_t(\Delta_t+2\xi_t)=(\Delta_t+\xi_t)^{2}-\xi_t^{2}\ge-\xi_t^{2}$, so the rest is at most
$\alpha^{2}s_t^{2}+\gamma_t\xi_t^{2}\le\alpha^{2}s_t^{2}+\tfrac{1}{8s_t^{2}}\alpha^{2}s_t^{4}
=\tfrac98\alpha^{2}s_t^{2}$.
\end{proof}

The two cases say: \emph{a revoked stretch pays for its own step}, since leaving the slab takes the
query a distance $4\alpha s_t^{2}$ along $v_t$ and the refund on the way out covers the cost
$\alpha^{2}s_t^{2}$ of the move that preceded the stretch; a surviving stretch leaves its
$\tfrac98\alpha^{2}s_t^{2}$ to be paid by the potential of the next lemma. The constants in
\eqref{eq:stretch} are the ones that make the first case balance exactly: with weight
$\gamma_t=\tau/s_t^{2}$ and a half-width $h_t$ with $\gamma_th_t=\alpha/2$, the refund margin is
$\tfrac{\alpha}{2}(h_t-2\alpha s_t^{2})=\alpha^{2}s_t^{2}(\tfrac{1}{4\tau}-1)$, which covers
$\alpha^{2}s_t^{2}$ exactly when $\tau\le\tfrac18$.

\subsection{The surviving stretches}\label{sec:pot}

It remains to bound $\sum_{t\in A_{T+1}}s_t^{2}$. The metric in use is not monotone, since stretches
come and go, so we look instead at the metric built from the surviving stretches alone,
\[
  \bar H_t\;:=\;I+\sum_{i\in A_{T+1},\ i<t}\gamma_iv_iv_i^{\top}\qquad(t=1,\dots,T+1) ,
\]
which is monotone, and use the potential $\tr(\bar H_t^{-1})$: it starts at $d$ and never goes
negative.

\begin{lemma}[Surviving stretches]\label{lem:potential}
$\displaystyle\sum_{t\in A_{T+1}}s_t^{2}\;\le\;9\bigl(d-\tr(\bar H_{T+1}^{-1})\bigr)\;\le\;9d$.
\end{lemma}

\begin{proof}
If $t\notin A_{T+1}$ then $\bar H_{t+1}=\bar H_t$. Fix $t\in A_{T+1}$ and put
$\bar s_t:=\|v_t\|_{\bar H_t^{-1}}$. A surviving stretch $i<t$ was made before round $t$ and is never
revoked, so it is active throughout round $t$; hence $i\in\hat A_t(\alpha)$, so
$\bar H_t\preceq\hat H_t(\alpha)$, and inverting, $\bar s_t\ge s_t$. By Sherman--Morrison,
$\bar H_{t+1}^{-1}=\bar H_t^{-1}-\tfrac{\gamma_t}{1+\gamma_t\bar s_t^{2}}\bar H_t^{-1}v_tv_t^{\top}\bar
H_t^{-1}$, so
\[
  \tr(\bar H_t^{-1})-\tr(\bar H_{t+1}^{-1})
  \;=\;\frac{\gamma_t\,\|\bar H_t^{-1}v_t\|^{2}}{1+\gamma_t\bar s_t^{2}}
  \;\ge\;\frac{\gamma_t\,\bar s_t^{4}}{1+\gamma_t\bar s_t^{2}}
  \;\ge\;\frac{\bar s_t^{2}/8}{1+1/8}\;=\;\frac{\bar s_t^{2}}{9}\;\ge\;\frac{s_t^{2}}{9} .
\]
The first inequality is Cauchy--Schwarz, $\|\bar H_t^{-1}v_t\|\ge\ip{v_t}{\bar H_t^{-1}v_t}=\bar
s_t^{2}$ as $\|v_t\|=1$; the second holds because the middle expression is increasing in $\gamma_t$
and $\gamma_t=1/(8s_t^{2})\ge1/(8\bar s_t^{2})$. Telescoping from $\tr(\bar H_1^{-1})=d$ gives the
first bound, and $\tr(\bar H_{T+1}^{-1})>0$ the second.
\end{proof}

\begin{remark}[The potential]\label{rem:potential}
Our potential is $\tr(\bar H^{-1})$, whereas the analysis of \cite{Cai26} uses the trace power
$\tr(H^{-1/2})$. The trace of the inverse is closer to optimal experimental design: if each stretch
$\gamma_iv_iv_i^{\top}$ is read as an observation along $v_i$, then $\bar H$ is an information
matrix and $\tr(\bar H^{-1})$ is its A-optimality criterion \cite{Kiefer74,Pukelsheim93}.
\end{remark}

\subsection{Proof of Theorem~\ref{thm:rvm}}\label{sec:proof}

By \eqref{eq:ledger} and Lemma~\ref{lem:lifetime},
\[
  \|p_{T+1}-\wstar\|^{2}-\|\wstar\|^{2}\;\le\;-\alpha\sum_{t\le T}r_t
  +\frac98\,\alpha^{2}\!\!\sum_{t\in A_{T+1}}\!\!s_t^{2} ,
\]
and Lemma~\ref{lem:potential} bounds the last sum by $9d$. With $\|\wstar\|\le R$ and
$\alpha=R/(3\sqrt d)$, so that $d\alpha^{2}=R^{2}/9$,
\[
  \alpha\,\Reg_T+\|p_{T+1}-\wstar\|^{2}\;\le\;R^{2}+\frac{81}{8}\,d\alpha^{2}
  \;=\;R^{2}+\frac98R^{2}\;=\;\frac{17}{8}R^{2} .
\]
Both terms on the left are nonnegative. Dropping the second gives
$\Reg_T\le\tfrac{17}{8}R^{2}/\alpha=\tfrac{51}{8}R\sqrt d$, and dropping the first gives
$\|p_{T+1}-\wstar\|^{2}\le\tfrac{17}{8}R^{2}$. \qed

\section{Further Properties}\label{sec:further}

We explain why revocation is needed and why it happens exactly on the slab boundary, and give a
variant with $O(d\log t)$ active stretches.

% Switched off at the author's request.
\iffalse
\subsection{Bounded iterates}\label{sec:bounded}

\begin{remark}[No projection]\label{rem:bounded}
The update has no projection step, and needs none: Theorem~\ref{thm:rvm} bounds
$\|p_{T+1}-\wstar\|\le\sqrt{17/8}\,R<1.46\,R$ for every $T$, so every query satisfies
$\|p_t\|<2.46\,R$. The algorithm keeps no domain; the iterates stay bounded on their own, as in
\cite{Cai26,SakaueSOP}. The queries need not lie in $K_1$, which the game allows.
\end{remark}
\fi

\subsection{Why revocation, and why on the boundary}\label{sec:why}

\begin{remark}[Stalling without revocation]\label{rem:stall}
The weight $\gamma_t=1/(8s_t^{2})$ is far heavier than the weight $\tau/s_t$, $\tau=\tfrac{1}{6d}$,
of \cite{Cai26}, and without revocation it would stop the query. On $K_1=R\Ball$ take $\wstar=Re_1$
and let the oracle answer $v_t=e_1$ in every round, which is legal as long as $\ip{p_t}{e_1}\le R$.
If no stretch were ever revoked, the metric would be $\mathrm{diag}(\mu_t,1,\dots,1)$ with
$s_t^{2}=1/\mu_t$ and $\mu_{t+1}=\mu_t+\gamma_t=\tfrac98\mu_t$, so round $t$ would advance the query
along $e_1$ by $\alpha/\mu_t=\alpha(\tfrac89)^{t-1}$, and all rounds together by at most
$9\alpha=3R/\sqrt d$. For $d\ge10$ the query would never reach $\wstar$, and every round would lose
$r_t\ge R(1-3/\sqrt d)$. With revocation, a stretch stays active only until the query has
advanced $4\alpha s_i^{2}$ past $p_{i+1}$, the point where it was made, so a direction that keeps
being confirmed does not keep accumulating weight, and Theorem~\ref{thm:rvm} bounds the total loss
by $\tfrac{51}{8}R\sqrt d$.
\end{remark}

\begin{remark}[Revocation exactly on the boundary]\label{rem:why}
Lemma~\ref{lem:lifetime} uses \eqref{eq:delta} in both directions, and each fails under a simpler
rule. Without revocation every stretch survives and nothing bounds $\Delta_t$: the refund
$\gamma_t\ip{\wstar-q_t}{v_t}^{2}$ is collected only at the end, where it can be zero, and then
$\Lambda_t$ contains the penalty $\gamma_t(r_t-\xi_t)^{2}\approx r_t^{2}/(8s_t^{2})$, which no
multiple of $-\alpha r_t$ absorbs when $r_t$ is large compared with $\alpha s_t^{2}$;
Remark~\ref{rem:stall} shows that this is not an artifact of the proof. Revoking a stretch before
the query leaves its slab gives $|\Delta_t|<4\alpha s_t^{2}$, and the refund need no longer cover
the cost $\alpha^{2}s_t^{2}$ of the move. Revoking it after the query leaves, as a single step with
the metric frozen followed by a check would do (the dotted segment in Figure~\ref{fig:round}),
gives $|\Delta_t|>4\alpha s_t^{2}$, and the linear part
$2\gamma_t\Delta_tr_t$ is no longer bounded by $\alpha r_t$. Nor does the general bound keep the
overshoot of a single step within a constant multiple of the half-width: the step moves the query
along $v_i$ by $\alpha|v_i^{\top}H_t^{-1}v_t|$, which Cauchy--Schwarz bounds only by
$\alpha\|v_i\|_{H_t^{-1}}\|v_t\|_{H_t^{-1}}\le\sqrt8\,\alpha s_is_t$, using
$\|v_i\|^{2}_{H_t^{-1}}\le1/(1+\gamma_i)<8s_i^{2}$ for $i\in A_t$ and $\|v_t\|_{H_t^{-1}}\le s_t$
(Lemma~\ref{lem:wellposed}(c)), and this exceeds the half-width $4\alpha s_i^{2}$ of $L_i$ by the
factor $s_t/(\sqrt2\,s_i)$, which is unbounded. This is why the move is split at the slab
boundaries.
\end{remark}

\subsection{Fewer active stretches}\label{sec:small}

In RVM$(\alpha)$ every round makes a stretch, and up to $t-1$ of them can be active at round $t$,
which is where the factor $t$ in the cost of Remark~\ref{rem:cost} comes from. A stretch with a tiny
$s_t$ carries a tiny slab and a tiny step, and skipping it costs almost nothing.

\begin{proposition}[Skipping small stretches]\label{prop:small}
Modify RVM$(\alpha)$ so that stretch $t$ is made only if $s_t^{2}>(t+1)^{-4}$, and otherwise
$A_{t+1}=\hat A_t(\alpha)$ and $H_{t+1}=\hat H_t(\alpha)$. With $\alpha=R/(3\sqrt d)$, for every
$T$, $\wstar\in K_1$ and strong separation oracle,
\[
  \Reg_T\;\le\;\Bigl(\frac{51}{8}+\frac{0.03}{d}\Bigr)R\sqrt d\;<\;6.41\,R\sqrt d ,
  \qquad
  |A_t|\;\le\;\frac{d\,\ln\bigl(1+t^{5}/(8d)\bigr)}{\ln(9/8)}\;=\;O(d\log t)\quad\text{for all }t .
\]
Round $t$ then costs $O(d^{3}\log^{2}t)$ arithmetic operations, and the first $T$ rounds cost
$O(Td^{2}\log T)$ in total.
\end{proposition}

\begin{proof}
\emph{Regret.} Lemmas~\ref{lem:wellposed} and~\ref{lem:path} do not refer to the rule for making
stretches, and neither do Lemmas~\ref{lem:lifetime} and~\ref{lem:potential}, which concern stretches
that are made. In the ledger \eqref{eq:ledger}, a round that makes no stretch contributes
$\Lambda_t=-2\alpha r_t+\alpha^{2}s_t^{2}\le-\alpha r_t+\alpha^{2}(t+1)^{-4}$, as
$s_t^{2}\le(t+1)^{-4}$. The proof of Theorem~\ref{thm:rvm} then gives
$\alpha\Reg_T\le\tfrac{17}{8}R^{2}+\alpha^{2}\sum_{t\ge1}(t+1)^{-4}$, and
$\sum_{n\ge2}n^{-4}=\tfrac{\pi^{4}}{90}-1<0.083$, so
$\Reg_T\le\tfrac{51}{8}R\sqrt d+0.083\,\alpha\le\tfrac{51}{8}R\sqrt d+0.028\,R/\sqrt d$.

\emph{Active stretches.} Fix $t$ and list $A_t=\{i_1<\dots<i_k\}$. Put $G_1=I$ and
$G_{j+1}=G_j+\gamma_{i_j}v_{i_j}v_{i_j}^{\top}$, so that $G_{k+1}=H_t$. Each $i_l$ with $l<j$ was made
before round $i_j$ and is still active at round $t$, so it was active throughout round $i_j$;
hence $G_j\preceq\hat H_{i_j}(\alpha)$ and $v_{i_j}^{\top}G_j^{-1}v_{i_j}\ge s_{i_j}^{2}$. By the matrix
determinant lemma and $\gamma_is_i^{2}=\tfrac18$,
\[
  \det H_t\;=\;\prod_{j=1}^{k}\bigl(1+\gamma_{i_j}v_{i_j}^{\top}G_j^{-1}v_{i_j}\bigr)
  \;\ge\;\Bigl(\frac98\Bigr)^{k} .
\]
On the other hand every stretch made has $\gamma_i=1/(8s_i^{2})<(i+1)^{4}/8$, so
$\tr H_t\le d+\tfrac18\sum_{i=1}^{t-1}(i+1)^{4}\le d+t^{5}/8$, and the inequality of arithmetic and
geometric means gives $\det H_t\le(\tr H_t/d)^{d}\le(1+t^{5}/(8d))^{d}$. Comparing the two bounds
gives the claim on $|A_t|$.

\emph{Cost.} By Remark~\ref{rem:cost}, a piece costs $O(d^{2}+|A_t|d)=O(d^{2}\log t)$, round $t$ has
at most $|A_t|+1$ pieces, and the first $T$ rounds have at most $2T$ pieces in total.
\end{proof}

\section{Conclusion and AI Disclosure}\label{sec:open}

We gave a deterministic proper learner for online inverse linear optimization with regret
$O(\sqrt d)$, the optimal order, using $\mathrm{poly}(d,T)$ arithmetic operations per round, and a
rounded variant that uses $\mathrm{poly}(d,\log t)$ bit operations per round
(Appendix~\ref{sec:bits}). It is
the variable-metric method of \cite{Cai26} with one change: stretches are revoked when the query
leaves their slab, which lets their weight be a constant multiple of $1/s_t^{2}$. The running time
still depends on $T$: the variant of Proposition~\ref{prop:small} takes $O(d^{2}\log T)$ operations
per round on average, the logarithm being the number of active stretches, whereas \cite{Cai26}
takes $O(d^{2})$ but has regret $O(dR)$. Whether regret $O(R\sqrt d)$ is possible with
$\mathrm{poly}(d)$ operations per round, independent of $T$, remains open.\\

The proof was initially obtained end to end by Colosseum \cite{colosseum}, a fully automated agentic
system built on the Gemini~4 base model. The raw output of Colosseum is available at
\href{https://drive.google.com/drive/folders/1LBHzHo29Yij__42lhnHFdYT8Dnjf9-LR?usp=drive_link}{\texttt{[link]}}, and a Lean~4 verification of all the results in this
paper at \href{https://github.com/renatoppl/oio-lean}{\texttt{[link]}}. The authors verified the proof and edited it for
clarity of presentation; other models were used later to edit the text.

The raw output goes beyond the main body in two directions. First, it bounds the bit complexity;
Appendix~\ref{sec:bits} presents this part, simplified and for uncorrupted feedback. Second, it
allows corrupted feedback. Against an oracle whose responses violate $r_t\ge0$ by a total of $C$,
that is, $\sum_t(r_t)_-\le C$ with $x_\pm:=\max(\pm x,0)$, RVM$(\alpha)$ has regret $\sum_{t\le
T}(r_t)_+=O(R\sqrt d+C)$ for every $T$, with no knowledge of $C$. The only step of our analysis that
uses $r_t\ge0$ is the bound on the part of $\Lambda_t$ linear in $r_t$ in Lemma~\ref{lem:lifetime};
without the sign, the same bound reads $-\alpha(r_t)_++3\alpha(r_t)_-$, and the rest of the proof
goes through unchanged. This brings the corruption-robust guarantee of \cite[Section~3.2]{Cai26},
which needs an upper bound on $C$, to the rate $\sqrt d$ and to unknown $C$. We chose not to
present this extension, and instead to focus on the core geometric idea of revocable stretches.

\appendix

\section{Polynomial Bit Complexity}\label{sec:bits}

The main body counts arithmetic operations on real numbers. Here we count bit operations, and show
that a rounded version of RVM keeps regret $O(R\sqrt d)$ with a number of bit operations per round
polynomial in $d$, $\log t$ and the encoding length of $R$; this answers the question of
\cite{SakaueMMW} in the bit model as well, relative to the interface for the replies described
below. Rounding RVM$(\alpha)$ as it stands would not do: its constants make the revoked case of
Lemma~\ref{lem:lifetime} balance exactly, which leaves no room for rounding errors. The rounded
version takes smaller weights and wider slabs, which make that room; it rounds the replies inward and
the velocities toward zero onto a dyadic grid, so that its queries are computed exactly; it revokes
a stretch anywhere in the outer half of its slab, rather than exactly on the boundary; and it stores
the inverse metric as the determinant and the adjugate of an integer matrix, so that no denominators
accumulate.

\paragraph{The model.} We count bit operations, with integer arithmetic done by the schoolbook
methods, so that multiplying or dividing an integer of $B$ bits by one of $C$ bits costs $O(BC)$.
The radius $R>0$ is a dyadic rational, given in binary by an integer mantissa and an integer exponent of
total length $\ell_R$, and the queries are output exactly, as dyadic vectors in the same form. (If
$R=0$, the query $0$ never loses.) The replies are vectors $v_t$ with $\|v_t\|\le1$, not
necessarily unit vectors, still subject to $r_t=\ip{\wstar-p_t}{v_t}\ge0$, and the learner reads
them through a coordinate interface: asked for precision $b\ge1$ in round $t$, it returns, at a cost
of $O(db)$ bit operations, an integer vector $m\in\Z^{d}$ with
\begin{equation}\label{eq:access}
  \|2^{-b}m-v_t\|_{\infty}\;\le\;2^{-b} .
\end{equation}
When the coordinates of $v_t$ are given explicitly as rationals, the interface rounds them, at an
additional cost polynomial in their encoding length.

\begin{theorem}[Bit complexity]\label{thm:bits}
Let $R>0$ be dyadic and $b_t:=10+\lceil\log_2d\rceil+6\lceil\log_2(t+1)\rceil=O(\log(d(t+1)))$.
For every $T$, every $\wstar$ with $\|\wstar\|\le R$ and every strong separation oracle whose replies
$v_t\in\Ball$ are read through \eqref{eq:access}, the rounded RVM of Definition~\ref{def:ralg} has
\[
  \Reg_T\;\le\;21\,R\sqrt d
  \qquad\text{and}\qquad
  |A_t|\;\le\;65\,d\ln\Bigl(1+\frac{t^{3}}{64d}\Bigr)\;=\;O(d\log(t+1))\quad\text{for all }t ,
\]
for the queries it outputs and the true replies. Its round $t$ uses
$O\bigl((d^{5}+d^{3}\log(t+1))\log(t+1)\,b_t^{2}+d\ell_Rb_t\bigr)$ bit operations and
$O\bigl((d^{3}+d^{2}\log(t+1))\,b_t+d\ell_R\bigr)$ bits of memory, and its first $T$ rounds use
$O\bigl((d^{4}+d^{2}\log(T+1))\,b_T^{2}+d\ell_Rb_T\bigr)$ bit operations per round on average.
\end{theorem}

The constant $21$ in place of the $\tfrac{51}{8}$ of Theorem~\ref{thm:rvm} is the price of the
slack. Through the reduction of Section~\ref{sec:intro}, with half-differences of actions as replies,
the theorem gives a proper learner for online inverse linear optimization in the bit model
(Corollary~\ref{cor:bits}).

\subsection{The rounded algorithm}\label{sec:ralg}

\begin{definition}[Rounded RVM]\label{def:ralg}
Let $\nu$ be the least power of two with $\nu\ge\sqrt d$, let $\alpha:=R/(8\nu)$, and let
$g_t:=2^{-b_t}$, with $b_t$ as in Theorem~\ref{thm:bits}. Start from $p_1=0$ and $A_1=\emptyset$. In
round $t=1,2,\dots$:
\begin{enumerate}
\itemsep0pt
\item Query $p_t$, request $v_t$ at precision $b_t$, receiving $m\in\Z^{d}$, and round it inward to
      $\tilde v_t:=g_tz_t$, where $z_{t,j}:=\sign(m_j)\max\{|m_j|-1,0\}$.
\item Starting at $p=p_t$ and $A=A_t$ with time $\alpha$ left, repeat the following \emph{motion}
      until no time is left. Let $H=I+\sum_{i\in A}\gamma_i\tilde v_i\tilde v_i^{\top}$ be the metric
      of the set $A$ of active stretches, and let $u$ be $H^{-1}\tilde v_t$ with each coordinate
      truncated toward zero to a multiple of $g_t$. Move $p$ to $p+\theta u$, where $\theta$ is the
      largest multiple of $\alpha g_t$ that is at most the time left and at most the time at which
      the ray $p+\theta'u$, $\theta'\ge0$, leaves the slab $L_i$ of some active stretch $i$. Then
      revoke every active stretch $i$ with $|\ip{p-p_{i+1}}{\tilde v_i}|\ge h_i/2$.
\item Set $p_{t+1}:=p$ and $s_t^{2}:=\tilde v_t^{\top}H^{-1}\tilde v_t$ for the metric $H$ left after
      the last motion. If $s_t^{2}>(t+1)^{-2}$, make stretch $t$, with
      \[
        \gamma_t:=\frac{\lceil1/s_t^{2}\rceil}{64},\qquad h_t:=\frac{\alpha}{2\gamma_t},\qquad
        L_t:=\bigl\{p\in\R^{d}:\ |\ip{p-p_{t+1}}{\tilde v_t}|\le h_t\bigr\},
      \]
      and set $A_{t+1}:=A\cup\{t\}$; otherwise set $A_{t+1}:=A$.
\end{enumerate}
\end{definition}

The motions of round $t$ trace a piecewise linear path $\hat p_t:[0,\alpha]\to\R^{d}$, and as in
\eqref{eq:move} we write $\hat A_t(\eta)$ and $\hat H_t(\eta)$ for the set of active stretches and
the metric in force at time $\eta$; they change only where a motion ends,
$H_t:=\hat H_t(0)=I+\sum_{i\in A_t}\gamma_i\tilde v_i\tilde v_i^{\top}$ is the metric at the start of
the round, and $s_t=\|\tilde v_t\|_{\hat H_t(\alpha)^{-1}}$ as in \eqref{eq:stretch}. The precision
makes the grid fine: $2^{b_t}\ge2^{10}d\,(t+1)^{6}$, so $g_t\le2^{-10}d^{-1}(t+1)^{-6}$. A motion
is the move of RVM$(\alpha)$ made discrete: its velocity is $\hat H_t(\eta)^{-1}\tilde v_t$ up to
the truncation, time advances in ticks of $\alpha g_t$, and the motion stops at the last tick before
the query would leave a slab, where every stretch in the outer half of its slab is revoked. The
weight $\gamma_t$ lies between $1/(64s_t^{2})$ and $1/(32s_t^{2})$, an eighth to a quarter of the
weight in RVM$(\alpha)$, and the product $\gamma_th_t=\alpha/2$ is the same as there, so the
half-width $h_t$ lies between $16\alpha s_t^{2}$ and $32\alpha s_t^{2}$, four to eight times as
wide. The threshold for making a stretch keeps every weight at most $(t+1)^{2}/64$, as the threshold
of Proposition~\ref{prop:small} does there.

Figure~\ref{fig:shell} shows a motion that ends in a revocation. The band is the slab $L_i$ of an
active stretch $i$, of half-width $h_i$ around $p_{i+1}$ along $\tilde v_i$, and its outer halves,
the \emph{shells}, are shaded darker. The motion starts at $p$, in the inner half, and advances in
ticks of $\alpha g_t$, the dots, drawn far longer than they are. It stops at $q_i$, the last tick
before the ray, dotted, leaves $L_i$. Since a tick is shorter than $h_i/2$, the point $q_i$ lies in
the shell, so stretch $i$ is revoked there, and the next motion starts in a new direction. Unlike
the move of RVM$(\alpha)$ in Figure~\ref{fig:round}, the revocation point lies only in the outer
half of the slab, not on its boundary; the slack in the weights and half-widths pays for this.

\begin{figure}[t]
\centering
\begin{tikzpicture}[scale=1.0,>={Stealth[length=5pt]}]
  % Slab of stretch i: half-width h_i = 1.6, inner half 0.8. One tick moves (0.7,0.3) from
  % (1.6,-0.3); the ray leaves the slab after 6.33 ticks, so the motion stops after 6, in the shell.
  \fill[okblue,fill opacity=0.07] (0,-0.8) rectangle (9,0.8);
  \fill[okblue,fill opacity=0.2] (0,0.8) rectangle (9,1.6);
  \fill[okblue,fill opacity=0.2] (0,-1.6) rectangle (9,-0.8);
  \draw[slabedge] (0,1.6) -- (9,1.6);
  \draw[slabedge] (0,-1.6) -- (9,-1.6);
  \draw[okblue,dash pattern=on 2pt off 1.5pt,line width=0.5pt] (0,0.8) -- (9,0.8);
  \draw[okblue,dash pattern=on 2pt off 1.5pt,line width=0.5pt] (0,-0.8) -- (9,-0.8);
  \draw[black!40,line width=0.4pt] (0,0) -- (9,0);
  \fill (0.7,0) circle (0.045);
  \node[lbl,anchor=north] at (0.7,-0.05) {$p_{i+1}$};
  \draw[decorate,decoration={brace,amplitude=3pt}] (-0.1,0) -- (-0.1,0.8)
    node[lbl,midway,left=3pt] {$h_i/2$};
  \draw[decorate,decoration={brace,amplitude=4pt}] (9.1,1.6) -- (9.1,0)
    node[lbl,midway,right=4pt] {$h_i$};
  \node[lbl,okblue] at (3.0,1.2) {shell};
  \node[lbl,okblue] at (3.0,-1.2) {shell};
  \node[lbl,okblue,anchor=south east] at (8.95,-1.58) {$L_i$};
  \draw[->,okblue,line width=0.6pt] (8.1,-0.55) -- (8.1,-0.05);
  \node[lbl,okblue,anchor=east] at (8.07,-0.3) {$\tilde v_i$};
  % the ray beyond the last tick, to the slab boundary and the first tick outside
  \draw[ghost] (5.8,1.5) -- (6.5,1.8);
  \draw[black!45,line width=0.5pt,fill=white] (6.5,1.8) circle (0.045);
  % the motion, in ticks, and the next motion after the revocation
  \draw[qpath] (1.6,-0.3) -- (5.8,1.5);
  \draw[qpath,->] (5.8,1.5) -- (6.8,2.5);
  \foreach \k in {1,...,5} \fill[black!75] ($(1.6,-0.3)+\k*(0.7,0.3)$) circle (0.04);
  \fill (1.6,-0.3) circle (0.045);
  \node[lbl,anchor=north west] at (1.62,-0.33) {$p$};
  \draw[line width=0.6pt,fill=white] (5.8,1.5) circle (0.06);
  \node[lbl,anchor=north west] at (5.85,1.44) {$q_i$};
\end{tikzpicture}
\caption{A motion of the rounded RVM, not to scale. It advances in ticks (dots) and stops at $q_i$,
the last tick before the ray (dotted) leaves the slab $L_i$, inside the darker outer half, where
stretch $i$ is revoked.}
\label{fig:shell}
\end{figure}
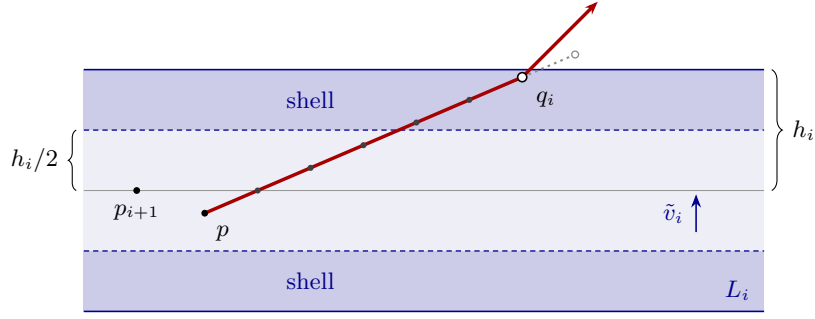

The next lemma plays the role of Lemma~\ref{lem:wellposed}. Besides the facts about the motions, it
records the size of the rounding errors and of the quantities they are multiplied by.

\begin{lemma}[Rounded move]\label{lem:rmove}
In every round $t$ of the rounded RVM:
\begin{enumerate}
\itemsep0pt
\item[(a)] $\|\tilde v_t\|\le\|v_t\|\le1$ and $\|\tilde v_t-v_t\|\le2\sqrt d\,g_t$, and the velocity
      $u$ of every motion satisfies $\|u\|\le1$ and $\|u-\hat H^{-1}\tilde v_t\|\le\sqrt d\,g_t$ for
      the metric $\hat H$ of that motion, and $s_t\le1$;
\item[(b)] every weight satisfies $\gamma_i\le(i+1)^{2}/64$, and for all $\eta$,
      $\|\hat p_t(\eta)\|\le\alpha t$, $\tr\hat H_t(\eta)\le d+t^{3}/64$ and
      $I\preceq\hat H_t(\eta)\preceq(t+1)^{3}I$;
\item[(c)] at the start of every motion, $|\ip{p-p_{i+1}}{\tilde v_i}|<h_i/2$ for every active
      stretch $i$;
\item[(d)] every motion lasts a positive multiple of $\alpha g_t$, and every motion but the last
      revokes a stretch, so the round has at most $k_t+1$ motions, where $k_t$ is the number of
      stretches it revokes;
\item[(e)] every stretch $i$ stays in $L_i$ while it is active, and if it is revoked at a point
      $q_i$ then $h_i/2\le|\ip{q_i-p_{i+1}}{\tilde v_i}|\le h_i$.
\end{enumerate}
\end{lemma}

\begin{proof}
(a) By \eqref{eq:access}, $|m_j-2^{b_t}v_{t,j}|\le1$. If $|m_j|\ge2$, then $z_{t,j}$ has the sign of
$v_{t,j}$ and $|z_{t,j}|=|m_j|-1\le2^{b_t}|v_{t,j}|$; if $|m_j|\le1$, then $z_{t,j}=0$. In both cases
$|\tilde v_{t,j}|\le|v_{t,j}|$ and $|\tilde v_{t,j}-v_{t,j}|\le2g_t$. Truncation toward zero moves
each coordinate by less than $g_t$ and does not increase its magnitude, and
$\|\hat H^{-1}\tilde v_t\|\le\|\tilde v_t\|\le1$ and $s_t^{2}\le\|\tilde v_t\|^{2}\le1$ as
$\hat H\succeq I$.

(b) If stretch $i$ is made, then $1/s_i^{2}<(i+1)^{2}$, so $64\gamma_i=\lceil1/s_i^{2}\rceil\le
(i+1)^{2}$. The stretches active in round $t$ were made in rounds $i<t$, with $\|\tilde v_i\|\le1$,
and $\sum_{i<t}(i+1)^{2}\le t^{3}$, which gives the bounds on $\hat H_t(\eta)$. By (a), every round
moves the query a distance at most $\alpha$, from $p_1=0$.

(c)--(e) A new stretch is made at the center $p_{t+1}$ of its slab, so it suffices to show that if
(c) holds at the start of a motion, then the motion satisfies (d) and (e) and (c) holds at its end
after the revocations. Let $u$ be its velocity, and for an active stretch $i$ let
$\phi_i:=\ip{p-p_{i+1}}{\tilde v_i}$ and $\beta_i:=\ip{u}{\tilde v_i}$. If $\beta_i\neq0$, the ray
leaves $L_i$ at time
\[
  \frac{h_i-\sign(\beta_i)\phi_i}{|\beta_i|}\;>\;\frac{h_i}{2}\;=\;\frac{16\alpha}{\lceil1/s_i^{2}\rceil}
  \;\ge\;\frac{16\alpha}{t^{2}}\;>\;\alpha g_t ,
\]
using $|\beta_i|\le1$ by (a), (b) for $i<t$, and $g_t<t^{-2}$. Hence the motion lasts a positive
multiple of $\alpha g_t$ and stays in every slab. If it is not the last motion of the round, it
ends less than $\alpha g_t$ before the ray leaves some $L_i$, at a point $p+\theta u$ where
$|\phi_i+\theta\beta_i|>h_i-\alpha g_t>h_i/2$, so stretch $i$ is revoked. The stretches revoked at
the end of the motion are those in the outer half of their slabs, which gives (e), and the others
satisfy (c). As the motions last multiples of $\alpha g_t$ that sum to $\alpha$, there are finitely
many, and (d) follows because a revoked stretch never returns.
\end{proof}

Parts (c)--(e) are the counterparts of Lemma~\ref{lem:wellposed}(a), (b) and (d): a motion never
leaves a slab, and it revokes a stretch only in the outer half of its slab, at a displacement
between $h_i/2$ and $h_i$ instead of exactly $h_i$.

\subsection{Regret}\label{sec:rregret}

We prove the bounds on $\Reg_T$ and $|A_t|$ in Theorem~\ref{thm:bits} by following
Section~\ref{sec:analysis} and pointing out what changes. Fix $T$, $\wstar$ and the oracle, and in
round $t$ put
\[
  \tilde r_t\;:=\;\ip{\wstar-p_t}{\tilde v_t},\qquad
  \xi(\eta)\;:=\;\ip{\hat p_t(\eta)-p_t}{\tilde v_t},\qquad \xi_t\;:=\;\xi(\alpha),
\]
the loss against the rounded reply and the distance moved along it. For every stretch $t$ made, the
point $q_t$ and the displacement $\Delta_t:=\ip{q_t-p_{t+1}}{\tilde v_t}$ are as in
Section~\ref{sec:move}, and the surviving stretches form the set $A_{T+1}$. As
$\|\wstar\|\le R=8\nu\alpha<16\sqrt d\,\alpha$, Lemma~\ref{lem:rmove}(b) gives
$\|\hat p_t(\eta)-\wstar\|\le\alpha(t+16\sqrt d)\le16\sqrt d\,(t+1)\alpha$ throughout round $t$,
and with Lemma~\ref{lem:rmove}(a),
\begin{equation}\label{eq:rtilde}
  |\tilde r_t-r_t|\;\le\;\|\wstar-p_t\|\,\|\tilde v_t-v_t\|\;\le\;32\,d\,(t+1)\,g_t\,\alpha .
\end{equation}

\emph{One move.} On a motion the metric $H=\hat H_t(\eta)$ is fixed and
$\hat p_t'=H^{-1}\tilde v_t+\epsilon$ with $\|\epsilon\|\le\sqrt d\,g_t$, so, as in the proof of
Lemma~\ref{lem:path},
\[
  \frac{d}{d\eta}\,\|\hat p_t(\eta)-\wstar\|_{H}^{2}
  \;=\;-2\bigl(\tilde r_t-\xi(\eta)\bigr)+2\ip{\hat p_t(\eta)-\wstar}{H\epsilon} .
\]
Here $\xi'(\eta)=\sigma(\eta)^{2}+\ip{\epsilon}{\tilde v_t}$, with
$\sigma(\eta)^{2}:=\tilde v_t^{\top}\hat H_t(\eta)^{-1}\tilde v_t\in[0,s_t^{2}]$ since the metric
only loses stretches during the round, so
\begin{equation}\label{eq:rxi}
  -\eta\sqrt d\,g_t\;\le\;\xi(\eta)\;\le\;\eta\bigl(s_t^{2}+\sqrt d\,g_t\bigr) .
\end{equation}
The last term of the derivative is at most $2\cdot16\sqrt d\,(t+1)\alpha\cdot(t+1)^{3}\cdot\sqrt
d\,g_t$ by Lemma~\ref{lem:rmove}(b). Integrating over $[0,\alpha]$ and subtracting the refund of each
stretch revoked in the round,
\[
  \|p_{t+1}-\wstar\|_{\hat H_t(\alpha)}^{2}-\|p_t-\wstar\|_{H_t}^{2}
  \;\le\;-2\alpha\tilde r_t+\alpha^{2}s_t^{2}+33\,d\,(t+1)^{4}g_t\,\alpha^{2}
  -\sum_{i\in A_t\setminus\hat A_t(\alpha)}\gamma_i\ip{\wstar-q_i}{\tilde v_i}^{2} .
\]

\emph{The ledger.} The penalty of stretch $t$ is $\gamma_t(\tilde r_t-\xi_t)^{2}$ and its refund
$\gamma_t(\tilde r_t-\xi_t-\Delta_t)^{2}$, so the ledger of Section~\ref{sec:move} holds with
$\tilde v_t$ and $\tilde r_t$ in place of $v_t$ and $r_t$, each round that makes no stretch
contributing its move alone:
\[
  \|p_{T+1}-\wstar\|^{2}-\|\wstar\|^{2}\;\le\;\sum_{t=1}^{T}\Lambda_t,
\]
where now $\Lambda_t:=-2\alpha\tilde r_t+\alpha^{2}s_t^{2}+33\,d\,(t+1)^{4}g_t\,\alpha^{2}$, plus
$\gamma_t\bigl(2\Delta_t\tilde r_t-\Delta_t(\Delta_t+2\xi_t)\bigr)$ if stretch $t$ is made.

\emph{The lifetime of a stretch.} Suppose that stretch $t$ is made. By Lemma~\ref{lem:rmove}(e),
$|\Delta_t|\le h_t$, and $\gamma_th_t=\alpha/2$, so the coefficient $-2\alpha+2\gamma_t\Delta_t$ of
$\tilde r_t$ lies in $[-3\alpha,-\alpha]$; as $r_t\ge0$, \eqref{eq:rtilde} bounds the part of
$\Lambda_t$ linear in $\tilde r_t$ by $-\alpha r_t+96\,d\,(t+1)g_t\,\alpha^{2}$. By \eqref{eq:rxi}
and $s_t^{2}>(t+1)^{-2}>\sqrt d\,g_t$, $|\xi_t|\le2\alpha s_t^{2}$, and
$h_t=32\alpha/\lceil1/s_t^{2}\rceil>16\alpha s_t^{2}$ since $\lceil1/s_t^{2}\rceil<2/s_t^{2}$. If
stretch $t$ is revoked, then $|\Delta_t|\ge h_t/2>8\alpha s_t^{2}$ by Lemma~\ref{lem:rmove}(e), and
as $a\mapsto a(a-4\alpha s_t^{2})$ is increasing for $a\ge2\alpha s_t^{2}$,
\[
  \gamma_t\Delta_t(\Delta_t+2\xi_t)\;\ge\;\gamma_t|\Delta_t|\bigl(|\Delta_t|-4\alpha s_t^{2}\bigr)
  \;\ge\;\gamma_t\,\frac{h_t}{2}\Bigl(\frac{h_t}{2}-4\alpha s_t^{2}\Bigr)
  \;=\;\frac{\alpha}{4}\Bigl(\frac{h_t}{2}-4\alpha s_t^{2}\Bigr)\;\ge\;\alpha^{2}s_t^{2} ,
\]
so the refund pays for the move, as in Lemma~\ref{lem:lifetime}. If stretch $t$ survives, completing
the square as there leaves $\alpha^{2}s_t^{2}+\gamma_t\xi_t^{2}\le\alpha^{2}s_t^{2}(1+4\gamma_t
s_t^{2})\le\tfrac98\alpha^{2}s_t^{2}$, as $\gamma_ts_t^{2}<\tfrac{1}{32}$. If no stretch is made,
\eqref{eq:rtilde} gives $-2\alpha\tilde r_t\le-\alpha r_t+64\,d\,(t+1)g_t\,\alpha^{2}$, and
$\alpha^{2}s_t^{2}\le\alpha^{2}(t+1)^{-2}$. The rounding errors add up to at most
$129\,d\,(t+1)^{4}g_t\,\alpha^{2}\le\tfrac{129}{1024}(t+1)^{-2}\alpha^{2}$, so in every case
\[
  \Lambda_t\;\le\;-\alpha r_t+\tfrac87\,\alpha^{2}(t+1)^{-2},\quad\text{plus }
  \tfrac98\,\alpha^{2}s_t^{2}\text{ if stretch $t$ is made and survives.}
\]

\emph{The surviving stretches.} The proof of Lemma~\ref{lem:potential} applies with $\tilde v_t$ in
place of $v_t$, using $\|\tilde v_t\|\le1$ in the Cauchy--Schwarz step. In its notation,
$\gamma_t\bar s_t^{2}\ge\gamma_ts_t^{2}\ge\tfrac{1}{64}$, so the decrease of the potential at a
surviving stretch $t$ is at least $\bar s_t^{2}\cdot\gamma_t\bar s_t^{2}/(1+\gamma_t\bar
s_t^{2})\ge\bar s_t^{2}/65\ge s_t^{2}/65$, and $\sum_{t\in A_{T+1}}s_t^{2}\le65d$.

\emph{Conclusion.} Summing over $t$, with $\|\wstar\|\le R$ and
$\sum_{t\ge1}(t+1)^{-2}=\pi^{2}/6-1<0.65$,
\[
  \alpha\,\Reg_T+\|p_{T+1}-\wstar\|^{2}\;\le\;R^{2}+\frac98\cdot65\,d\alpha^{2}
  +\frac87\cdot0.65\,\alpha^{2}\;\le\;R^{2}+74\,d\alpha^{2} .
\]
With $\rho:=\nu/\sqrt d\in[1,2)$, this gives
$\Reg_T\le R^{2}/\alpha+74d\alpha=(8\rho+\tfrac{37}{4\rho})R\sqrt d$. The coefficient is convex in
$\rho$, so it is at most the larger of its values at the endpoints, $\tfrac{69}{4}$ at $\rho=1$ and
$\tfrac{165}{8}<21$ at $\rho=2$.

\emph{Active stretches.} The proof of Proposition~\ref{prop:small} applies with $\tilde v_i$ in place
of $v_i$: each factor of the determinant is at least $1+\gamma_is_i^{2}\ge\tfrac{65}{64}$, and
$\tr H_t\le d+t^{3}/64$ by Lemma~\ref{lem:rmove}(b). Hence
$(\tfrac{65}{64})^{|A_t|}\le\det H_t\le(1+t^{3}/(64d))^{d}$, and the bound on $|A_t|$ follows from
$\ln\tfrac{65}{64}\ge\tfrac{1}{65}$. \qed

\subsection{Exact integer arithmetic}\label{sec:integer}

We now run the rounded RVM in integer arithmetic, with no rounding beyond the truncations of
Definition~\ref{def:ralg}, and bound the lengths of the integers involved. The precisions $b_t$ are
nondecreasing in $t$, so $g_{t-1}$ is a multiple of $g_t$. In round $t$
every motion moves the query by $\theta u$ with $\theta\in\alpha g_t\Z$ and $u\in g_t\Z^{d}$, so by
induction, throughout round $t$, $p=\alpha2^{-2b_t}Y$ for an integer vector $Y$: the queries are
computed exactly. Write $b:=b_t$ and $\mu:=2b+6$. An active stretch $i$, made in round $i$, has
$\tilde v_i=2^{-b_i}z_i$ and $p_{i+1}=\alpha2^{-2b_i}Y_{i+1}$ for integer vectors $z_i$ and
$Y_{i+1}$, and an integer $\kappa_i:=64\gamma_i=\lceil1/s_i^{2}\rceil$; we write
$\bar z_i:=2^{b-b_i}z_i$ and $\bar Y_{i+1}:=4^{b-b_i}Y_{i+1}$ for these vectors at the current
precision. The metric is then $H=2^{-\mu}G$ for the integer matrix
\[
  G\;:=\;2^{\mu}I+\sum_{i\in A}\kappa_i\,\bar z_i\bar z_i^{\top} ,
\]
and the algorithm stores, in place of $H^{-1}$, the integer $D:=\det G$ and the integer matrix
$J:=\adj G=DG^{-1}$, so that $H^{-1}=2^{\mu}J/D$. Every stretch made or revoked changes $G$ by a
rank-one term, and $D$ and $J$ follow by the next lemma, a fraction-free form of Sherman--Morrison.

\begin{lemma}[Adjugate update]\label{lem:adj}
Let $G$ and $G'=G\pm\kappa zz^{\top}$ be positive definite integer matrices, with $\kappa\in\Z$
and $z\in\Z^{d}$, and let $D=\det G$, $J=\adj G$ and $y=Jz$. Then
\[
  \det G'\;=\;D\pm\kappa\ip{z}{y},\qquad
  \adj G'\;=\;\frac{(\det G')\,J\mp\kappa\,yy^{\top}}{D},
\]
and the division is exact.
\end{lemma}

\begin{proof}
By the matrix determinant lemma, $\det G'=D(1\pm\kappa z^{\top}G^{-1}z)=D\pm\kappa\ip{z}{y}$. By
Sherman--Morrison,
$G'^{-1}=G^{-1}\mp\kappa\,G^{-1}zz^{\top}G^{-1}/(1\pm\kappa z^{\top}G^{-1}z)$; multiplying by
$\det G'$ and substituting $G^{-1}=J/D$ gives the formula for $\adj G'=(\det G')\,G'^{-1}$. The
adjugate of an integer matrix is an integer matrix, so the division is exact.
\end{proof}

Starting from $Y=0$, $A=\emptyset$, $D=2^{d\mu}$ and $J=2^{(d-1)\mu}I$ at $\mu=2b_1+6$, and with
$b_0:=b_1$, round $t$ runs as follows.
\begin{enumerate}
\itemsep0pt
\item Output $p_t=\alpha2^{-2b_{t-1}}Y$. Multiply $Y$ by $4^{b-b_{t-1}}$, $D$ by $4^{d(b-b_{t-1})}$
      and $J$ by $4^{(d-1)(b-b_{t-1})}$, as the change of precision multiplies $G$ by
      $4^{b-b_{t-1}}$. Request $v_t$, compute $z_t$, and set $N:=2^{b}$, the number of ticks left.
\item While $N>0$, compute $y:=Jz_t$ and $U_j:=\sign(y_j)\lfloor2^{\mu}|y_j|/D\rfloor$, so that
      $u=g_tU$, since $H^{-1}\tilde v_t=g_t2^{\mu}y/D$. For each active $i$ compute
      $\Phi_i:=\ip{Y-\bar Y_{i+1}}{\bar z_i}$ and $\Psi_i:=\ip{U}{\bar z_i}$, so that
      $\ip{p-p_{i+1}}{\tilde v_i}=\alpha2^{-3b}\Phi_i$ and $\ip{u}{\tilde v_i}=2^{-2b}\Psi_i$; as
      $h_i=32\alpha/\kappa_i$, if $\Psi_i\neq0$ the ray leaves $L_i$ after
      $(2^{3b+5}-\sign(\Psi_i)\kappa_i\Phi_i)/(\kappa_i|\Psi_i|)$ ticks. Let $n$ be the least of $N$
      and the integer parts of these numbers, set $Y\leftarrow Y+nU$ and $N\leftarrow N-n$, and
      revoke, one at a time, every active $i$ with $\kappa_i|\Phi_i+n\Psi_i|\ge2^{3b+4}$, updating
      $D$ and $J$ by Lemma~\ref{lem:adj} with the minus sign, $\kappa=\kappa_i$ and $z=\bar z_i$.
\item Compute $S:=\ip{z_t}{Jz_t}$, so that $s_t^{2}=64S/D$. If $64(t+1)^{2}S>D$, make stretch
      $t$: store $z_t$, $Y_{t+1}:=Y$, $b_t$ and $\kappa_t:=\lceil D/(64S)\rceil$, and update $D$ and
      $J$ by Lemma~\ref{lem:adj} with the plus sign, $\kappa=\kappa_t$ and $z=z_t$.
\end{enumerate}
The only roundings are those of Definition~\ref{def:ralg}: the inward rounding of the reply, the
floors in $U$ and in the numbers of ticks, and the ceiling in $\kappa_t$. Every matrix $G$ met along
the way, including those between the revocations of one motion, is $2^{\mu}I$ plus positive
semidefinite terms, so Lemma~\ref{lem:adj} applies.

\paragraph{Bit lengths and cost.} We prove the cost bounds of Theorem~\ref{thm:bits}. Throughout
round $t$, Lemma~\ref{lem:rmove} gives $\|p\|\le\alpha t$, hence $|Y_j|\le4^{b}t$, and likewise for
$\bar Y_{i+1}$; also $|\bar z_{i,j}|\le2^{b}$, $|U_j|\le2^{b}$ and $\kappa_i\le(t+1)^{2}$. As
$\log_2(dt)\le b$, these integers, the $\Phi_i$, the $\Psi_i$ and the numerators of the numbers of
ticks have $O(b)$ bits.
Since $G=2^{\mu}H\succeq2^{\mu}I$, and $\tr H\le d+(t+1)^{3}/64$ even after stretch $t$ is made, the
inequality of arithmetic and geometric means gives $2^{d\mu}\le D\le2^{d\mu}(1+(t+1)^{3}/(64d))^{d}$,
and $|J_{jk}|\le\|J\|=D\|G^{-1}\|\le2^{-\mu}D$; so $D$, the entries of $J$ and $y$, $S$ and the
integers in Lemma~\ref{lem:adj} have $O(db)$ bits. A motion therefore costs $O(d^{3}b^{2})$ bit
operations for $y$ and $U$ and $O(|A_t|db^{2})$ for the scan of the active stretches, each use of
Lemma~\ref{lem:adj} costs $O(d^{4}b^{2})$, and the change of precision and the request cost
$O(d^{3}b)$. Writing $p_t$ in mantissa--exponent form multiplies the mantissa of $R$ by $d$ integers
of $O(b)$ bits, at a cost of $O(d\ell_Rb)$. By Lemma~\ref{lem:rmove}(d), round $t$ thus uses
\[
  O\bigl((k_t+1)(d^{4}+d|A_t|)\,b_t^{2}+d\ell_Rb_t\bigr)
\]
bit operations, and the counts of Theorem~\ref{thm:bits} follow from $k_t\le|A_t|=O(d\log(t+1))$ and
$\sum_{t\le T}k_t\le T$. The memory holds $D$ and $J$, in $O(d^{3}b_t)$ bits, the $|A_t|$ active
stretches, in $O(db_t)$ bits each, and the query, in $O(d(\ell_R+b_t))$ bits. \qed

\subsection{Inverse optimization in the bit model}\label{sec:bitsoilo}

For online inverse linear optimization we use the reduction of Section~\ref{sec:intro} with one
change: the reply is the half-difference $\tfrac12(x_t-\hat x_t)$ rather than the normalized
difference, which avoids the square root and needs no precision beyond that of the actions.

\begin{corollary}[Inverse optimization in the bit model]\label{cor:bits}
Suppose that the actions are read through the coordinate interface \eqref{eq:access}. There is a
deterministic proper learner for online inverse linear optimization with $R_T\le42\sqrt d$ for every
$T$, every $\wstar\in\Ball$ and every adaptive sequence of action sets. Its round $t$ uses one
linear optimization over $X_t$, for an integer objective with $O(b_t)$ bits per coordinate, and
$O((d^{5}+d^{3}\log(t+1))\log(t+1)\,b_t^{2})$ bit operations, and its first $T$ rounds use
$O((d^{4}+d^{2}\log(T+1))\,b_T^{2})$ bit operations per round on average.
\end{corollary}

\begin{proof}
Run the rounded RVM with $R=1$. Before seeing $X_t$, take as objective the integer vector $Y$ with
$p_t=\alpha2^{-2b_{t-1}}Y$, or $e_1$ if $Y=0$, recommend the action $\hat x_t\in X_t$ returned by the
linear optimization, and after observing $x_t$ reply $v_t:=\tfrac12\delta_t$ with
$\delta_t:=x_t-\hat x_t$. As $Y$ is a positive multiple of $p_t$, optimality of $\hat x_t$ and $x_t$
gives $\ip{p_t}{\delta_t}\le0\le\ip{\wstar}{\delta_t}$, so $\|v_t\|\le1$,
$r_t=\tfrac12\ip{\wstar-p_t}{\delta_t}\ge0$, and the shortfall is
$\ip{\wstar}{\delta_t}\le\ip{\wstar-p_t}{\delta_t}=2r_t$. Theorem~\ref{thm:bits} then gives
$R_T\le2\Reg_T\le42\sqrt d$; every round, including those with $\delta_t=0$, is a round of the game.
To answer a request of precision $b$, read $x_t$ and $\hat x_t$ at precision $b+2$, receiving $m$
and $\hat m$; then $(m_j-\hat m_j)/8$ is within $\tfrac14$ of $2^{b}v_{t,j}$, and rounding it to a
nearest integer answers \eqref{eq:access} at a cost of $O(db)$.
\end{proof}

\end{document}